\documentclass[sn-mathphys-num,pdflatex]{sn-jnl}% Math and Physical Sciences Numbered Reference Style 

\usepackage{amsmath,amssymb,amsfonts}%
\usepackage{amsthm}%
\usepackage{tikz}
\usetikzlibrary{arrows,shapes,arrows.meta}
\usepackage{lineno}

\usepackage{todonotes}

\newtheoremstyle{pln}
  {0.5em} % Space above
  {0.5em} % Space below
  {\it} % Body font
  {} % Indent amount
  {\bfseries} % Theorem head font
  {.} % Punctuation after theorem head
  {0.5em} % Space after theorem head
  {} % Theorem head spec (can be left empty, meaning `normal')

\theoremstyle{pln}%
\newtheorem{theorem}{Theorem}%  meant for continuous numbers
\newtheorem{proposition}[theorem]{Proposition}% 
\newtheorem{lemma}[theorem]{Lemma}

\newtheoremstyle{defn}
  {0.5em} % Space above
  {0.5em} % Space below
  {} % Body font
  {} % Indent amount
  {\bfseries} % Theorem head font
  {.} % Punctuation after theorem head
  {0.5em} % Space after theorem head
  {} % Theorem head spec (can be left empty, meaning `normal')
\theoremstyle{defn}
\newtheorem{definition}{Definition}%
\newtheorem{remark}{Remark}%

\newcommand{\normI}[1]{\Vert #1 \Vert}
\newcommand{\normD}[1]{\left\Vert #1 \right\Vert}
\newcommand{\bfz}{\mathbf{0}}
\newcommand{\bfs}{\mathbf{s}}

\begin{document}

\title[Smoothness Assumptions and Examples]{Recent Advances in Non-convex Smoothness Conditions and Applicability to Deep Linear Neural Networks}


\author*[1]{\fnm{Vivak} \sur{Patel}}\email{vivak.patel@wisc.edu}

\author[1]{\fnm{Christian} \sur{Varner}}\email{cvarner@wisc.edu}

\affil*[1]{\orgdiv{Department of Statistics}, \orgname{University of Wisconsin}, \orgaddress{\street{1300 University Ave}, \city{Madison}, \postcode{53706}, \state{WI}, \country{USA}}}


\abstract{The presence of non-convexity in smooth optimization problems arising from deep learning have sparked new smoothness conditions in the literature and corresponding convergence analyses. We discuss these smoothness conditions, order them, provide conditions for determining whether they hold, and evaluate their applicability to training a deep linear neural network for binary classification.}




\maketitle

\section{Introduction}\label{sec-intro}
For smooth optimization problems, non-convexity has typically been addressed by one of two common strategies. 
One strategy is to focus on optimization problems for which either the gradient function or Hessian function is globally Lipschitz continuous \citep[Assumption f.gL and f.HL]{cartis2022Evaluationa}.

\begin{definition}[Globally Lipschitz Continuous] \label{def-global}
Let $(\mathcal{X}, d_\mathcal{X}),(\mathcal{Y},d_\mathcal{Y})$ be metric spaces. A continuous function, $D:\mathcal{X} \to \mathcal{Y}$, is globally Lipschitz continuous if there exists a constant $C \geq 0$ such that $d_\mathcal{Y}(D(x_1),D(x_2)) \leq C d_\mathcal{X}(x_1, x_2)$ for all $x_1,x_2 \in \mathcal{X}$. The constant $C$ is referred to as the global Lipschitz constant.
\end{definition}
\begin{remark}
Given an objective function, $f:\mathbb{R}^p\to\mathbb{R}$, $\mathcal{X} = \mathbb{R}^p$. As examples, $\mathcal{Y} = \mathbb{R}^p$ if we are considering the gradient function to be globally Lipschitz continuous; and $\mathcal{Y} = \mathbb{R}^{p\times p}$ if we are considering the Hessian function to be globally Lipschitz continuous.
\end{remark}

For functions whose gradient or Hessian functions satisfy this condition, optimization methods typically enjoy global guarantees for converging to a first-order or second-order stationary point [\citenum{cartis2022Evaluationa}, Theorems 2.2.2 and 3.1.1; \citenum{beck2017Firstorder}, Lemma 10.4; \citenum{bertsekas2016Nonlinear}, Proposition 1.2.4; \citenum{nocedal2006Numerical}, Theorem 3.2].
In the second (less common) strategy, the class of problems considered are those for which the gradient function or Hessian function is level-set Lipschitz continuous \citep[Exercise 1.2.5]{bertsekas2016Nonlinear}.

\begin{definition}[Level-set Lipschitz Continuous] \label{def-level-set}
Let $f:\mathbb{R}^p \to \mathbb{R}$. Let $L(s) = \lbrace x \in \mathbb{R}^p: f(x) \leq s \rbrace$ for every $s \in \mathbb{R}$. Suppose $D$, a (higher-order) derivative of $f$, exists and is continuous everywhere in $\mathbb{R}^p$. $D$ is level-set Lipschitz continuous if for every $s \in \mathbb{R}$ such that $L(s) \neq \emptyset$, there exists $C_s \geq 0$ such that
\begin{equation}
	\normD{D(x_1) - D(x_2)} \leq C_s \normD{x_1 - x_2} ~ \forall x_1, x_2 \in L(s).\footnote{As norms are equivalent in finite-dimensional vector spaces, the choice of norms can be arbitrary here.}
\end{equation}
\end{definition}

The class of functions whose gradient is level-set Lipschitz continuous clearly contains those whose gradient is globally Lipschitz continuous.
For example, the class of functions whose gradient is level-set Lipschitz continuous includes $f(x) = x^{2k+2}$ for all $k \in \mathbb{N}$, yet these monomials are \textit{not} in the class of functions whose gradient is globally Lipschitz continuous (see Lemma \ref{result-equivalence-lipschitz-derivative}).
Despite this ordering, these two function classes have never had much distinction as historically popular optimization methods \textit{enforce descent}, ensuring that the regions explored by the iterates have a common (worst-case) Lipschitz constant \citep[e.g.,][Theorem 3.2]{nocedal2006Numerical}.

With a resurgence of methods that \textit{do not enforce descent} to address non-convex optimization problems for (smooth) deep learning, most analyses require that a derivative function be globally Lipschitz continuous (see examples in Table \ref{table-objective-free}).
However, the class of functions for which the gradient function is globally Lipschitz continuous was shown to be inapplicable to certain smooth optimization problems arising in machine learning, such as training one-dimensional feed forward networks with three hidden layers and training one-dimensional recurrent neural networks with three time steps \citep[Propositions 1 to 2]{patel2022Global}.

\begin{table}[ht]
\centering
\caption{Example convergence results for optimization methods that require global Lipschitz continuity of a derivative function.}
\label{table-objective-free}
\footnotesize
\begin{tabular}{p{2in}p{2in}} \toprule
\textbf{Method} & \textbf{Convergence} \\ \midrule
Diminishing Step-Size & \cite[Proposition 1.2.4]{bertsekas2016Nonlinear} \\
Constant Step-Size & \cite[Corollary 1]{armijo1966Minimization} \\
Barzilai-Borwein Methods & \cite[\S 4]{barzilai1988Twopoint}, \cite[Theorem 3.3]{burdakov2019Stabilized} \\
Nesterov's Acceleration Method & \cite[Theorem 6]{nesterov2013Gradient} \\
Bregman Distance Method & \cite[Theorem 1]{bauschke2017Descent} \\
Negative Curvature Method & \cite[Theorem 1]{curtis2019Exploiting} \\
Lipschitz Approximation & \cite[Theorem 1]{malitsky2020Adaptive} \\
Weighted Gradient-Norm Damping & \cite[Theorem 2.3]{wu2020Wngrad}, \cite[Corollary 1]{grapiglia2022Adaptive} \\
Adaptively Scaled Trust Region & \cite[Theorem 3.10]{gratton2022Convergence} \\ \botrule
\end{tabular}
\end{table}

A natural generalization that accounts for such problems is to require that the needed derivative function is locally Lipschitz continuous \citep[e.g.,][Assumption 2.2]{patel2024Gradient}.
\begin{definition}[Locally Lipschitz Continuous] \label{def-local}
Let $(\mathcal{X}, d_\mathcal{X}),(\mathcal{Y},d_\mathcal{Y})$ be metric spaces. A  continuous function $D:\mathcal{X} \to \mathcal{Y}$ is locally Lipschitz continuous if, for every $x \in \mathcal{X}$, there exists an open ball around $x$, denoted $B_x$, and there exists a constant $C_{B_x} \geq 0$ such that $d_\mathcal{Y}(D(x_1),D(x_2)) \leq C_{B_x} d_\mathcal{X}(x_1, x_2)$ for all $x_1,x_2 \in B_x$.
\end{definition}

Unfortunately, under this condition, optimization methods can behave counterintuitively: they can produce iterates whose optimality gap diverges and whose gradients remain bounded away from zero [\citenum{patel2024Gradient}, Proposition 4.4]. Indeed, such behavior is shown to occur for all of the optimization methods in Table \ref{table-objective-free} using carefully constructed objective functions whose gradient functions are locally Lipschitz continuous [\citenum{varner2024Challenges}, Proposition 3.3 to 3.11].

Fortunately, new smoothness conditions have appeared under which certain optimization methods \textit{that do not enforce descent} are shown to behave as desired \cite{chen2023GeneralizedSmooth,li2023Convex}. These new smoothness conditions fill in some of the gaps between global Lipschitz continuity and local Lipschitz continuity. 
Our goal here is to discuss these smoothness conditions used in convergence analyses (\S\ref{sec-conditions});
put them in an ordering by generality (\S\ref{sec-ordering});
develop relationships between smoothness conditions and differentiability (\S\ref{sec-differentiability});
and evaluate their applicability to training deep linear neural networks (\S\ref{sec-examples}). 

For training deep linear neural networks for binary classification, we show that these recently proposed conditions do not apply, and only Definition \ref{def-local} is appropriate for the gradient function.
Owing to this result, we recommend that these recently developed smoothness conditions be verified for a given function class---such as those arising for training more complex deep neural networks---before being assumed to hold in the analysis of an optimization method on this function class.

\section{Smoothness Conditions} \label{sec-conditions}
In general, new smoothness conditions allow the local Lipschitz constant to grow at a certain rate that may depend on properties of the objective function or its derivatives. Stemming from the works of \citep{zhang2019Why,zhang2020Improved,li2023Convex}, one new smoothness condition occurs when the gradient function is $\rho$-order Lipschitz continuous.

\begin{definition}[$\rho$-Order Lipschitz Continuous] \label{defn-rho-order}
Let $(\mathcal{X},\normI{\cdot}_\mathcal{X}), (\mathcal{Y}, \normI{\cdot}_{\mathcal{Y}})$ be normed vector spaces.
Let $\rho \geq 0$.
A continuous function, $D: \mathcal{X} \to \mathcal{Y}$, is $\rho$-order Lipschitz continuous if there exist constants $C_0, C_1 \geq 0$ and, for every $x \in \mathcal{X}$, there exist an open ball around $x$, $B_x$, such that
	\begin{equation}
		\normD{ D(x_1) - D(x_2) }_\mathcal{Y} \leq (C_0 + C_1 \normD{ D(x)}_\mathcal{Y}^\rho) \normD{ x_1 - x_2 }_\mathcal{X} ~\forall x_1,x_2 \in B_x.
	\end{equation}
\end{definition}

Stemming from \citep{chen2023GeneralizedSmooth}, another smoothness condition coincides with the gradient function being $\rho$-integrated Lipschitz continuous.
\begin{definition}[$\rho$-Integrated Lipschitz Continuous] \label{defn-rho-integrated}
Let $(\mathcal{X},\normI{\cdot}_\mathcal{X}), (\mathcal{Y}, \normI{\cdot}_{\mathcal{Y}})$ be normed vector spaces.
Let $\rho \geq 0$.
A continuous function $D: \mathcal{X} \to \mathcal{Y}$ is $\rho$-integrated Lipschitz continuous if there exist constants $C_0, C_1 \geq 0$ such that, $\forall x_1, x_2 \in \mathcal{X}$,
	\begin{equation}
		\normD{ D(x_1) - D(x_2) }_\mathcal{Y} \leq \left( C_0 + C_1 \int_0^1 \normD{D(x_1 + t(x_2 - x_1))}_\mathcal{Y}^\rho dt \right) \normD{ x_1 - x_2}_{\mathcal{X}}.
	\end{equation}
\end{definition}

Before we move onto specifying the relationships between the different choices of Lipschitz continuities, we note that there are two other smoothness condition stemming from our previous work \citep{patel2022Global,varner2023Novel}. In the first, the smoothness condition can only be expressed for a gradient-based algorithm where the angle between the search direction and gradient function can be well-controlled \citep[Assumption 5]{patel2022Global}, which is too specialized for the discussion here. In the second, the smoothness condition is more general than those here (aside from Definition \ref{def-local}) \citep[Theorem 3.10, Part 3]{varner2023Novel}, but requires a procedure that enforces descent, which is not within the scope of this work.

\section{Ordering} \label{sec-ordering}
To roughly summarize this section, we show
\begin{equation} \label{eqn-ordering}
\mathrm{Globally} \underset{(I)}{\subset} \rho\mathrm{-integrated} \underset{(II)}{=} \rho\mathrm{-order} \underset{(III)}{\subset} \mathrm{Locally},
\end{equation}
where the inclusions are strict.

\textbf{(I)} For any function satisfying Definition \ref{def-global} with global Lipschitz constant $C \geq 0$, the function satisfies Definition \ref{defn-rho-order} with $C_0 = C$ and $C_1 = 0$ for all $\rho \geq 0$.
Furthermore, we show in Proposition \ref{result-example-notGlobal-rhoOrder} that, for every $\rho > 0$, there exists a function that satisfies Definition \ref{defn-rho-order} but not Definition \ref{def-global}.%
\footnote{The case of $\rho = 0$ for Definition \ref{defn-rho-order} coincides with Definition \ref{def-global}.}

\textbf{(II)} We prove the following result, which generalizes \citep[Theorem 1]{chen2023GeneralizedSmooth}.

\begin{proposition} \label{result-order-integrated-equivalent}
Let $(\mathcal{X},\normI{\cdot}_\mathcal{X}), (\mathcal{Y}, \normI{\cdot}_{\mathcal{Y}})$ be normed vector spaces over $\mathbb{R}$ or $\mathbb{C}$. Let $D: \mathcal{X} \to \mathcal{Y}$ be continuous and let $\rho \geq 0$. $D$ is $\rho$-order Lipschitz continuous if and only if it is $\rho$-integrated Lipschitz continuous.
\end{proposition}

\begin{proof}[Proof of sufficiency of Proposition \ref{result-order-integrated-equivalent}]
Suppose $D$ is $\rho$-integrated Lipschitz continuous. 
Let $\epsilon > 0$.
Then, by the continuity of $z \mapsto \normI{D(z)}_2^\rho$, for each $x \in \mathcal{X}$, there exists an open ball, $B_x$, such that $\normI{D(z)}_\mathcal{Y}^\rho \leq \normI{D(x)}_\mathcal{Y}^\rho + \epsilon$ for all $z \in B_x$. Then, for any $x_1,x_2 \in B_x$,
\begin{small}
\begin{align}
	\normD{ D(x_2) - D(x_1) }_{\mathcal{Y}} &\leq \left( C_0 + C_1 \int_0^1 \normD{ D(x_1 + t(x_2 - x_1))}_\mathcal{Y}^\rho dt \right) \normD{x_2 - x_1}_\mathcal{X} \\
	&\leq (C_0 + C_1 \normD{ D(x)}_\mathcal{Y}^\rho + \epsilon) \normD{x_2 - x_1}_\mathcal{X}.
\end{align}
\end{small}
Thus, $D$ is $\rho$-order Lipschitz continuous.\footnote{We can try to optimize the choice of constants $C_0, C_1$ for the purposes for proving rate-of-convergence results. However, as rates of convergence are not the focus of this work, we will not discuss optimizing these constants.}
\end{proof}

To prove necessity, we will need some machinery around partitions (as is needed for a Riemann integral). A finite partition $P$ of $[0,1]$ is a collection of points, $\lbrace t_i : i=0,\ldots,|P| \rbrace$, such that $0=t_0 < t_1 < \cdots < t_{|P|-1} < t_{|P|} = 1$. 
The index set of $P$, $I_P$, is $\lbrace 0,1,\ldots,|P|-1 \rbrace$.%
\footnote{We do not include $|P|$ in $I_P$ as we will always need pairs of points $i \in I_P$ and $i+1$.}
The mesh of $P$ is $M_P = \max_{i \in I_P} t_{i+1} - t_i$. 
Two finite partitions $P,P'$ have the relationship $P \subset P'$ if for every $t \in P$, $t \in P'$ and $|P| < |P'|$, and we say $P'$ is a refinement of $P$.

Owing to the properties of Definition \ref{defn-rho-order}, we will need to be able to find a refinement of any partition $P$ in the following manner.

\begin{lemma} \label{result-partition-completion}
Let $(\mathcal{X},d_{\mathcal{X}})$ be a metric space; $x_1,x_2 \in \mathcal{X}$; $z(t) = x_1 + t(x_2-x_1) \in \mathcal{X}$ for $t \in [0,1]$; and let $L = \lbrace z(t) : t\in[0,1] \rbrace$.
Furthermore, suppose for every point $z \in \mathcal{L}$ there exists a fixed open ball $B_z$. 
Let $P$ be a finite partition of $[0,1]$. Then, there exists a finite partition 
$P^*$ that refines $P$ such that either $z(t_i) \in B_{z(t_{i+1})}$ or $z(t_{i+1}) \in B_{z(t_i)}$ for all $i \in I_{P^*}$ and $t_i,t_{i+1} \in P^*$.
\end{lemma}
\begin{proof}
Let $j \in I_P$ such that $z(s_j) \not\in B_{z(s_{j+1})}$ and $z(s_{j+1}) \not\in B_{z(s_{j})}$ where $s_j,s_{j+1} \in P$. $\lbrace B_z(t): t \in [s_j, s_{j+1}] \rbrace$ is an open cover of a compact set, for which there exists a finite sub-cover that corresponds to a partition of $[s_j, s_{j+1}]$ that we denote by $Q^j$. Now, suppose $k \in I_{Q^j}$ such that $z(r_k) \not\in B_{z(r_{k+1})}$ and $z(r_{k+1}) \not\in B_{z(r_k)}$ where $r_k, r_{k+1} \in Q^j$. Since $\lbrace z(t) : t \in [s_j, s_{j+1}] \rbrace$ is connected, $B_{z(r_k)} \cap B_{z(r_{k+1})} \neq \emptyset$. We then add a time point to $Q^j$ that corresponds to a $z \in \lbrace z(t) : t \in (r_k, r_{k+1}) \rbrace \cap B_{z(r_k)} \cap B_{z(r_{k+1})}$. Since the initial sub-cover is finite, the final partition of $[s_j, s_{j+1}]$, $Q^j$, is finite, and, for $t \in I_{Q^j}$, $z(r_t) \in B_{z(r_{t+1})}$ and/or $z(r_{t+1}) \in B_{z(r_t)}$. Finally, we consider the partition $P$ in union with all $Q^j$ for $j \in I_P$ such that $z(s_j) \not\in B_{z(s_{j+1})}$ and $z(s_{j+1}) \not\in B_{z(s_{j})}$ where $s_j,s_{j+1} \in P$. This partition then has the desired property.
\end{proof}

\begin{proof}[Proof of necessity of Proposition \ref{result-order-integrated-equivalent}]
Suppose $D$ is $\rho$-order Lipschitz continuous. Let $x_1, x_2 \in \mathcal{X}$. Let $z(t) = x_1 + t(x_2 - x_1)$ for $t \in [0,1]$. Let $L = \lbrace z(t) : t\in [0,1] \rbrace$. For each $t \in [0,1]$, there exists a $B_{z(t)}$ corresponding to Definition \ref{defn-rho-order}. 
By Lemma \ref{result-partition-completion}, we can construct a sequence of partitions $\lbrace P_n \rbrace$ where $P_n \subset P_{n+1}$ with $\lim_{n} M_{P_n} = 0$ 
and, for every $i \in I_{P_n}$, $z(t_i) \in B_{z(t_{i+1})}$ and/or $z(t_{i+1}) \in B_{z(t_i)}$. Define $z^*(t_i) = z(t_i)$ if $z(t_{i+1}) \in B_{z(t_i)}$ and $z^*(t_i) = z(t_{i+1})$ if $z(t_{i}) \in B_{z(t_{i+1})}$. Then,
\begin{small}
\begin{align}
	\normD{ D(x_2) - D(x_1) }_{\mathcal{Y}} &\leq \sum_{i \in I_{P_n}} \normD{ D(z(t_{i+1})) - D(z(t_i))}_\mathcal{Y} \\
										   &\leq \sum_{i \in I_{P_n}} (C_0 + C_1 \normD{ D(z^*(t_i))}_\mathcal{Y}^\rho ) \normD{ z(t_{i+1}) - z(t_i)}_\mathcal{X} \\
										   &\leq \left(C_0 + C_1\sum_{i \in I_{P_n}} \normD{D(z^*(t_i))}_\mathcal{Y}^\rho (t_{i+1} - t_i) \right) \normD{ x_2 - x_1 }_\mathcal{X}. 
\end{align}
\end{small}

Since $D$ is a continuous function on $L$, the sum converges to its corresponding Riemann integral as $n \to \infty$. Hence, $D$ is $\rho$-integrated Lipschitz continuous.
\end{proof}

\textbf{(III)} For any function satisfying Definition \ref{defn-rho-order} with $C_0, C_1 \geq 0$ and $\rho \geq 0$ in each ball $B_x$, the function satisfies Definition \ref{def-local} with $C_{B_x} = C_0 + C_1 \normI{D(x)}_\mathcal{Y}^\rho$ in each ball $B_x$. 
Furthermore, we show in \S \ref{sec-examples} that the gradient function for training a deep linear neural network does not satisfy Definition \ref{defn-rho-order} for any $\rho, C_0, C_1 \geq 0$, yet does satisfy Definition \ref{def-local}. In other words, the class of functions satisfying Definition \ref{def-local} strictly contains the class of function satisfying Definition \ref{defn-rho-order} for all $\rho \geq 0$.

\section{Continuity Conditions and Differentiability} \label{sec-differentiability}
We now develop conditions related to differentiability that allow us to check whether a gradient function is globally Lipschitz continuous or $\rho$-order Lipschitz continuous / $\rho$-integrated Lipschitz continuous.
Our first sufficient and necessary condition extends the well-known result that a twice continuous differentiable function has a globally Lipschitz continuous gradient function if its Hessian is bounded.%
\footnote{Our result complements Proposition 3.2 of \cite{li2023Convex}. Proposition 3.2 of \cite{li2023Convex} allows for a more general relationship on the growth of the Lipschitz constant, whereas we allow for cases in which the constants are zero.}

\begin{lemma} \label{result-equivalence-lipschitz-derivative}
Let $f:\mathbb{R}^p \to \mathbb{R}$ be twice continuously differentiable. Let $\rho \geq 0$. $\nabla f(x)$ is $\rho$-order Lipschitz continuous if and only if $\exists C_0, C_1 \geq 0$ such that $\normI{\nabla^2 f(x) }_2 \leq C_0 + C_1 \normI{ \nabla f(x)}_2^\rho$, $\forall x \in \mathbb{R}^p$.
\end{lemma}
\begin{proof}
Let $\rho \geq 0$. Suppose $\exists C_0, C_1 \geq 0$ such that $\forall x \in \mathbb{R}^p$, there exists an open ball, $B_x$, containing $x$, for which
$\normI{\nabla f(x_1) - \nabla f(x_2)}_2 \leq (C_0 + C_1\normI{\nabla f(x)}_2^\rho) \normI{ x_1 - x_2 }_2$, $\forall x_1, x_2 \in B_x$.
Since $f$ is twice continuously differentiable, $\nabla^2 f$ is symmetric. Hence, there exists a unit vector $v \in \mathbb{R}^p$ such that
\begin{align}
\normD{ \nabla^2 f(x) }_2 
&= |v^\intercal \nabla^2 f(x) v| = \left\vert \lim_{t\downarrow 0} \frac{v^\intercal (\nabla f(x+tv) - \nabla f(x))}{t} \right\vert \\
&\leq \lim_{t \downarrow 0} \frac{\normD{ \nabla f(x+tv) - \nabla f(x) }_2}{t},
\end{align}
where we have made use of $\normI{v}_2 = 1$.
For $t$ sufficiently small, $x + tv, x \in B_x$ and 
\begin{equation}
\normI{\nabla f(x+tv) - \nabla f(x) }_2 \leq (C_0 + C_1 \normI{\nabla f(x)}_2^\rho)\normI{tv}_2 = (C_0 + C_1 \normI{\nabla f(x)}_2^\rho)t.
\end{equation}
Hence, $\normI{ \nabla^2 f(x) }_2 \leq C_0 + C_1 \normI{ \nabla f(x)}_2^\rho$.

Now, suppose $\exists C_0, C_1 \geq 0$ such that $\normI{ \nabla^2 f(x) }_2 \leq C_0 + C_1 \normI{ \nabla f(x)}_2^\rho$ for all $x\in \mathbb{R}^p$. Let $\epsilon > 0$. By the continuity of $x \mapsto \normI{\nabla^2 f(x)}_2$, for every $x \in \mathbb{R}^p$, there exist an open ball around every $x$, denoted $B_x$, such that $\normI{\nabla^2 f(y) }_2 \leq \normI{\nabla^2 f(x)}_2 + \epsilon$ for all $y \in B_x$. 
Thus, using the preceding observation with the fundamental theorem of calculus, for any $x_1, x_2 \in B_x$, 
\begin{align}
\normD{ \nabla f(x_1) - \nabla f(x_2) }_2 
&= \normD{ \int_0^1 \nabla^2f(x_2 + t (x_1 - x_2) ) (x_1 - x_2) dt}_2 \\
&\leq \normD{x_1 - x_2}_2\int_0^1 \normD{\nabla^2 f(x_2 + t (x_1 - x_2) )}_2 dt \\
&\leq \normD{x_1 - x_2}_2\left( \normD{ \nabla^2 f(x)}_2 + \epsilon \right) \\
&\leq \normD{x_1 - x_2}_2\left( C_0 + C_1 \normD{ \nabla f(x)}_2^\rho + \epsilon  \right).
\end{align}
Hence, $\nabla f(x)$ is $\rho$-order Lipschitz continuous.%
\footnote{Again, we can optimize the choice of constants $C_0, C_1$ if we are interested in a rate-of-converge result. As this is not our focus, we will not do so here.}
\end{proof}

Using this relationship, we can construct functions whose gradient is $\rho$-order Lipschitz continuous ($\rho > 0$), but not globally Lipschitz continuous.
\begin{proposition} \label{result-example-notGlobal-rhoOrder}
Let $\rho > 0$.
There exists a twice continuously differentiable function, $f: \mathbb{R} \to \mathbb{R}$, such that $f'(x)$ is $\rho$-order Lipschitz continuous but not globally Lipschitz continuous. 
\end{proposition}
\begin{proof}
We split the constructions into two parts. For $\rho \geq 1$, consider
\begin{small}
\begin{equation}
f(x) = \begin{cases}
e^{|x|} & |x| \geq 1 \\
\frac{e}{2}x^2 + \frac{e}{2} & |x| < 1,
\end{cases}
~
f'(x) = \begin{cases}
\mathrm{sign}(x) e^{|x|} & |x| \geq 1 \\
ex & |x| < 1,
\end{cases}
~
f''(x) = \begin{cases}
e^{|x|} & |x| \geq 1 \\
e & |x| < 1.
\end{cases}
\end{equation}
\end{small}%
It can be readily checked that $f$ is twice continuously differentiable. Now, $|f''(x)| \to \infty$ as $|x| \to \infty$. Hence, $f'$ is \textit{not} globally Lipschitz continuous by Lemma \ref{result-equivalence-lipschitz-derivative}. Furthermore, for any $\rho \geq 1$,
\begin{equation}
|f''(x)| \leq e + |f'(x)|^\rho.
\end{equation}
Hence $f'$ is $\rho$-order Lipschitz continuous by Lemma \ref{result-equivalence-lipschitz-derivative}.

For $\rho \in (0,1)$, consider a smooth version of the example in \citep[Theorem 1]{chen2023GeneralizedSmooth}:
\begin{small}
\begin{equation}
f(x) = \begin{cases}
\frac{1-\rho}{2-\rho}|x|^{\frac{2-\rho}{1-\rho}} & |x| \geq 1 \\
\frac{\rho}{8(1-\rho)} x^4 + \frac{2-3\rho}{4(1-\rho)} x^2 + \frac{1-\rho}{2-\rho} - \frac{4 - 5\rho}{8(1-\rho)} & |x| < 1,
\end{cases}
\end{equation}
\begin{equation}
f'(x) = \begin{cases}
\mathrm{sign}(x) |x|^{\frac{1}{1-\rho}} & |x| \geq 1 \\
\frac{\rho}{2(1-\rho)}x^3 + \frac{2-3 \rho}{2(1-\rho)} x & |x| < 1,
\end{cases}
~
f''(x) = \begin{cases}
\frac{1}{1-\rho} |x|^{\frac{\rho}{1-\rho}} & |x| \geq 1 \\
\frac{3\rho}{2(1-\rho)}x^2 + \frac{2-3 \rho}{2(1-\rho)} & |x| < 1.
\end{cases}
\end{equation}
\end{small}%
It is easy to check that $f$ is twice continuously differentiable. $f''(x)$ is unbounded and so $f'$ is \textit{not} globally Lipschitz continuous by Lemma \ref{result-equivalence-lipschitz-derivative}. Furthermore,
\begin{equation}
|f''(x)| \leq \frac{1}{1-\rho} + \frac{1}{1-\rho} |f'(x)|^\rho,
\end{equation}
which implies that $f'$ is $\rho$-order Lipschitz continuous by Lemma \ref{result-equivalence-lipschitz-derivative}.
\end{proof}

Finally, we provide a sufficient condition for demonstrating that a function is \textit{not} $\rho$-order/integrated Lipschitz continuous, which will be useful in \S \ref{sec-examples}.
\begin{lemma} \label{result-local-only}
Let $f:\mathbb{R}^p \to \mathbb{R}$ be twice continuously differentiable. Let $\rho \geq 0$. If $\exists \lbrace x^n \rbrace \subset \mathbb{R}^p$ and $\lbrace \kappa_n \rbrace \subset \mathbb{R}_{\geq 0}$ such that
\begin{enumerate}
\item $\lbrace \kappa_n \rbrace$ diverges;
\item $\liminf_n \kappa_n \normI{ \nabla f(x^n)}_2^\rho = \infty$; and
\item $\normI{ \nabla^2 f(x^n)}_2 \geq \kappa_n \normI{ \nabla f(x^n)}_2^\rho$;
\end{enumerate}
then $\nabla f$ is not $\rho$-order Lipschitz continuous for any $C_0, C_1 \geq 0$.
\end{lemma}
\begin{proof}
Fix $\rho \geq 0$. For a contradiction by Lemma \ref{result-equivalence-lipschitz-derivative}, suppose $\exists C_0, C_1 \geq 0$ such that, $\forall x \in \mathbb{R}^p$, 
$\normI{\nabla^2 f(x)}_2 \leq C_0 + C_1\normI{\nabla f(x)}_2^\rho$.
By the hypotheses of the statement, $\exists \lbrace x^n \rbrace \subset\mathbb{R}^p$ and a diverging sequence $\lbrace \kappa_n \rbrace \subset \mathbb{R}_{\geq 0}$
such that $\liminf_n \kappa_n\normI{\nabla f(x^n)}_2 = \infty$ and $\normI{\nabla^2 f(x^n)}_2 \geq \kappa_n \normI{ \nabla f(x^n)}_2^\rho$. Using this with the contradiction supposition, $\kappa_n \normI{\nabla f(x^n)}_2^\rho \leq \normI{ \nabla^2 f(x^n)}_2 \leq C_0 + C_1 \normI{ \nabla f(x^n) }_2^\rho$. Rearranging, $(\kappa_n - C_1) \normI{\nabla f(x^n)}_2^\rho \leq C_0$. The left hand side of the preceding inequality diverges to infinity by the hypotheses, yet the right remains bounded. Hence, we have a contradiction.
\end{proof}
\begin{remark}
We note that the choice of norms here is arbitrary as all norms are equivalent in finite-dimensional vector spaces. Below, we will use other norms to simplify the calculations.
\end{remark}

\section{Applicability to Deep Linear Neural Networks} \label{sec-examples}
We will begin by showing that an objective function for training a linear one-dimensional feed forward network described in \citep{patel2022Global} with three hidden layers for binary classification does not have a gradient that is $\rho$-order/integrated Lipschitz continuous for any $\rho \geq 0$. We then show that this result extends to training a multi-dimensional linear feed forward network of arbitrary depth.

\subsection{A Three Hidden Layer Neural Network}

Consider using a feed forward neural network to predict a binary label ($y$) from a one-dimensional input ($z$). Specifically, we consider a feed forward network with three hidden layers consisting of a one-dimensional weight ($x_1, x_2, x_3$) and identity activation function ($\sigma$), and a sigmoid output neuron ($\psi$) with a one-dimensional weight ($x_4$) as shown in Figure \ref{figure-ffn}. Note, this network does not have bias terms, which we will consider in the subsequent example.

\begin{figure}[hb]
\centering 
\caption{A diagram of a simple feed forward network with three hidden layers and an output layer for binary classification.}
\label{figure-ffn}
\begin{tikzpicture}
\node[circle,minimum width=0.5cm,fill=red!20,draw=black,dashed] (feat) at (0,0) {$z$};
\node[circle,minimum width=0.75cm,fill=blue!20,draw=black] (ly1) at (2,0) {$\sigma$};
\node[circle,minimum width=0.75cm,fill=blue!20,draw=black] (ly2) at (4,0) {$\sigma$};
\node[circle,minimum width=0.75cm,fill=blue!20,draw=black] (ly3) at (6,0) {$\sigma$};
\node[circle,minimum width=0.5cm,fill=green!20,draw=black,thick] (out) at (8,0) {$\varphi$};
\node (y) at (9,0) {$\hat y$};

\draw[->,thick] (feat) to [out=45,in=135] node[above] {$x_1\times$} (ly1);
\draw[->,thick] (ly1) to [out=45,in=135] node[above] {$x_2\times$} (ly2);
\draw[->,thick] (ly2) to [out=45,in=135] node[above] {$x_3\times$} (ly3);
\draw[->,thick] (ly3) to [out=45,in=135] node[above] {$x_4 \times$} (out);
\draw[->,thick] (out) to (y);
\end{tikzpicture}
\end{figure}
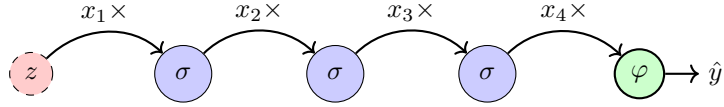

For the data, suppose we have input-output pairs, $(z,y)$, given by $(0,0)$ and $(1,1)$, where each occurs with equal probability. If we train such a network using the standard binary cross-entropy loss, then the resulting risk minimization problem's objective function is (up to a scaling and shift)
\begin{equation} \label{eqn-3-layer}
f(x) = \log(  1 + \exp(-x_1 x_2 x_3 x_4)),
\end{equation}
where $x = (x_1, x_2, x_3, x_4)$ \citep[see][for details]{patel2022Global}.

Then, 
\begin{equation}
\nabla f(x) = \frac{-1}{1 + \exp(x_1x_2x_3x_4)} 
\begin{bmatrix}
x_2 x_3 x_4 \\
x_1 x_3 x_4 \\
x_1 x_2 x_4 \\
x_1 x_2 x_3
\end{bmatrix},
\end{equation}
and
\begin{small}
\begin{equation}
\begin{aligned}
\nabla^2 f(x) &= \frac{-1}{1 + \exp(x_1 x_2 x_3 x_4)}
\begin{bmatrix}
0 & x_3 x_4 & x_2x_4 & x_2 x_3 \\
x_3x_4 & 0 & x_1 x_3 & x_1 x_3 \\
x_2 x_4 & x_1 x_4 & 0 & x_1 x_2 \\
x_2 x_3 & x_1 x_3 & x_1 x_2 & 0 
\end{bmatrix} \\
&+ \frac{\exp(x_1 x_2 x_3 x_4)}{[1 + \exp(x_1 x_2 x_3 x_4)]^2}
\begin{bmatrix}
x_2 x_3 x_4 \\
x_1 x_3 x_4 \\
x_1 x_2 x_4 \\
x_1 x_2 x_3
\end{bmatrix}
\begin{bmatrix}
x_2 x_3 x_4 \\
x_1 x_3 x_4 \\
x_1 x_2 x_4 \\
x_1 x_2 x_3
\end{bmatrix}^\intercal.
\end{aligned}
\end{equation}
\end{small}

\begin{lemma}
Let $f:\mathbb{R}^4 \to \mathbb{R}$ be defined as in (\ref{eqn-3-layer}). Then, $\nabla f: \mathbb{R}^4 \to \mathbb{R}^4$ is locally Lipschitz continuous.
\end{lemma}
\begin{proof}
$\nabla^2 f(x)$ is continuous for all $x \in \mathbb{R}^4$. Therefore, $\nabla f$ is locally Lipschitz continuous \citep[e.g.,][Lemma SM2.1]{patel2024Gradient}.
\end{proof}

To show that $\nabla f$ is \textit{not} $\rho$-order/integrated Lipschitz continuous, we will make use of Lemma \ref{result-local-only}. To begin, we compute the norms of the gradient and Hessian. 
\begin{equation}
\normD{\nabla f(x)}_2 = \frac{\sqrt{(x_1x_2x_3)^2 + (x_1 x_2 x_4)^2 + (x_1 x_3 x_4)^2 + (x_2 x_3 x_4)^2}}{1+\exp(x_1x_2x_3x_4)}.
\end{equation}

By the equivalence of the operator norm and the $(1,1)$ entry-wise norm ($\normI{ \cdot }_2 \geq 0.5 \normI{ \cdot }_{(1,1)}$ for a $4 \times 4$ matrix), we can use the $(1,1)$ entry-wise norm in our analysis. We lower bound the $(1,1)$ entry-wise norm using the matrix's $(1,1)$ entry, which yields
\begin{equation}
\begin{aligned}
\normD{\nabla^2 f(x)}_{(1,1)} &\geq \frac{\exp(x_1x_2x_3x_4)}{(1 + \exp(x_1x_2x_3x_4))^2}(x_2 x_3 x_4)^2.
\end{aligned}
\end{equation}

\begin{proposition} \label{result-3-layer}
Let $f:\mathbb{R}^4 \to \mathbb{R}$ be as in (\ref{eqn-3-layer}). $\forall \rho \geq 0$, $\nabla f(x)$ is neither $\rho$-integrated Lipschitz continuous nor $\rho$-order Lipschitz continuous.
\end{proposition}
\begin{proof}
We now show that $\nabla f$ is \textit{not} $\rho$-order Lipschitz continuous for any $\rho \geq 0$.
Let $\lbrace x^n = (n^{-1}, 1, 1, n\log(n)) \rbrace$ and 
\begin{equation}
\kappa_n = \frac{n \log(n)^{2}}{4 (2\log(n))^\rho}.
\end{equation}
For this choice of parameters, for $n > 1$,
\begin{equation}
\normD{\nabla f(x^n)}_2 = \frac{\sqrt{ n^{-2} + 2 \log(n)^2 + n^2 \log(n)^2}}{1 + n} \in [ 0.5\log(n), 2 \log(n)];
\end{equation}
and, for all $n \in \mathbb{N}$,
\begin{equation}
\normD{\nabla^2 f(x^n)}_{(1,1)} \geq \frac{n}{(1 + n)^2}n^2 \log(n)^2 \geq \frac{n \log(n)^2}{4}.
\end{equation}

We check the hypotheses of Lemma \ref{result-local-only} (for $n > 1$).
\begin{enumerate}
\item $\lim_n \kappa_n = \infty$.
\item $\kappa_n \normI{ \nabla f(x^n) }_2^\rho \geq  \kappa_n 0.5^\rho \log(n)^\rho = n\log(n)^2 / 4^{1+\rho}$. Hence, $\lbrace \kappa_n \normI{ \nabla f(x^n) }_2^\rho \rbrace$ diverges.
\item $ \normI{ \nabla^2 f(x^n) }_{(1,1)} \geq n\log(n)^2/4 = \kappa_n 2^\rho \log(n)^\rho \geq \kappa_n \normI{ \nabla f(x^n) }_2^\rho$.
\end{enumerate}
Thus, $\nabla f$ is not $\rho$-order Lipschitz continuous for any $\rho \geq 0$. 
\end{proof}

\subsection{An Arbitrary-Depth Neural Network}

Consider now using a feed forward neural network to predict a binary label ($y$) from an input $z \in \mathbb{R}^d$. As shown in Figure \ref{figure-dffn}, we use a network of $N-1$ hidden layers with identity activation functions ($\sigma$), arbitrary weights ($W_1,\ldots,W_{N-1} \in \mathbb{R}^{d \times d}$) and arbitrary biases ($b_1,\ldots,b_{N-1} \in \mathbb{R}^d$); followed by a sigmoid output layer with a weight ($W_N \in \mathbb{R}^{1 \times d}$) and bias ($b_N \in \mathbb{R}$). 

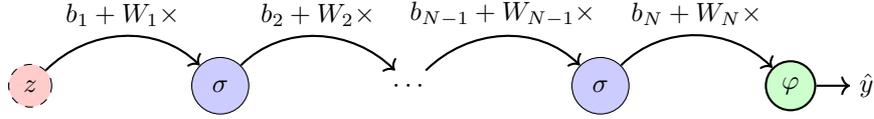
\begin{figure}[hb]
\centering 
\caption{A diagram of an arbitrary-depth, multi-dimensional feed forward network for binary classification.}
\label{figure-dffn}
\begin{tikzpicture}
\node[circle,minimum width=0.5cm,fill=red!20,draw=black,dashed] (feat) at (0,0) {$z$};
\node[circle,minimum width=0.75cm,fill=blue!20,draw=black] (ly1) at (2.5,0) {$\sigma$};
\node[] (ly2) at (5,0) {$\cdots$};
\node[circle,minimum width=0.75cm,fill=blue!20,draw=black] (ly3) at (7.5,0) {$\sigma$};
\node[circle,minimum width=0.5cm,fill=green!20,draw=black,thick] (out) at (10,0) {$\varphi$};
\node (y) at (11,0) {$\hat y$};

\draw[->,thick] (feat) to [out=45,in=135] node[above] {$b_1 + W_1\times$} (ly1);
\draw[->,thick] (ly1) to [out=45,in=135] node[above] {$b_2 + W_2\times$} (ly2);
\draw[->,thick] (ly2) to [out=45,in=135] node[above] {$b_{N-1} + W_{N-1} \times $} (ly3);
\draw[->,thick] (ly3) to [out=45,in=135] node[above] {$b_{N} + W_N \times$} (out);
\draw[->,thick] (out) to (y);
\end{tikzpicture}
\end{figure}

For this network, we denote the argument for the sigmoid output layer, $\psi$, by
\begin{small}
\begin{align}
\eta_z &= W_{N} (W_{N-1} (\cdots(W_2 (W_1 z + b_1) + b_2) \cdots)   + b_{N-1}) + b_N \\
	    &= \left(\prod_{i=1}^{N} W_i \right) z + \sum_{j=1}^{N} \left( \prod_{i=j+1}^{N} W_i \right) b_{j},
\end{align}
\end{small}%
where we take the ordering of the products to be from the largest index term on the left to the smallest index term on the right; and we use the convention that if the upper limit of the product is smaller than the lower limit of the product, then the product is set to the identity.

For the data, suppose we have two input-output pairs each occurring with probability $1/2$. For the first input-output pair, $z$ is $\mathbf{0} \in \mathbb{R}^d$ and $y = 0$; for the second input-output pair, $z$ is the first standard basis vector in $\mathbb{R}^d$, denoted $\mathbf{s}_1$, and $y = 1$. If we train the network on this data with binary cross entropy loss, then the objective function of the expected risk minimization problem is
\begin{small}
\begin{equation} \label{eqn-N-layer}
\begin{aligned}
&f(W_1,\ldots,W_{N}, b_1,\ldots,b_N) = \frac{1}{2}\left[ \eta_\bfz + \log( 1 + \exp(-\eta_\bfz)) \right] + \frac{1}{2}\log( 1 + \exp(-\eta_{\bfs_1})).
\end{aligned}
\end{equation}
\end{small}

We can also compute the gradient and Hessian of this function by the chain rule. First, note that $\eta_z$ is composed of twice continuously differentiable functions, which implies that $\eta_z$ is twice continuously differentiable. Then, by the chain rule,
\begin{small}
\begin{equation} \label{eqn-N-layer-partial-W}
\frac{\partial f}{\partial W_{\ell}(k,r)} = \frac{1}{2} \frac{1}{1+\exp(-\eta_\bfz)} \frac{\partial \eta_\bfz}{\partial W_{\ell}(k,r)} - \frac{1}{2} \frac{1}{1 + \exp(\eta_{\bfs_1})} \frac{\partial \eta_{\bfs_1}}{\partial W_\ell(k,r)},~\mathrm{and}
\end{equation}
\begin{equation} \label{eqn-N-layer-partial-b}
\frac{\partial f}{\partial b_{\ell}(r)}  = \frac{1}{2} \frac{1}{1+\exp(-\eta_\bfz)} \frac{\partial \eta_\bfz}{\partial b_{\ell}(r)} - \frac{1}{2} \frac{1}{1 + \exp(\eta_{\bfs_1})} \frac{\partial \eta_{\bfs_1}}{\partial b_{\ell}(r)},
\end{equation}
\end{small}%
where $W_{\ell}(k,r)$ is the $(k,r)$ entry of $W_{\ell}$ and $b_{\ell}(r)$ is the $r$ entry of $b_{\ell}$; and $W_{N}(k,r)$ always has $k=1$ since $W_N \in \mathbb{R}^{1 \times d}$.
Furthermore,
\begin{footnotesize}
\begin{equation}
\begin{aligned}
&\frac{\partial^2 f}{\partial W_{\ell_1}(k_1,r_1) \partial W_{\ell_2}(k_2,r_2)} \\
&= \frac{1}{2}\left[ \frac{1}{1+\exp(-\eta_\bfz)} \frac{\partial^2 \eta_\bfz}{\partial W_{\ell_1}(k_1,r_1) W_{\ell_2}(k_2,r_2)} + \frac{\exp(-\eta_\bfz)}{(1 + \exp(-\eta_\bfz))^2} \frac{\partial \eta_\bfz}{\partial W_{\ell_1}(k_1,r_1)}\frac{\partial \eta_\bfz}{\partial W_{\ell_2}(k_2,r_2)}  \right]  \\
&- \frac{1}{2}\left[\frac{1}{1 + \exp(\eta_{\bfs_1})} \frac{\partial^2 \eta_{\bfs_1}}{\partial W_{\ell_1}(k_1,r_1) W_{\ell_2}(k_2,r_2)} 
- \frac{\exp(\eta_{\bfs_1})}{(1 + \exp(\eta_{\bfs_1}))^2} \frac{\partial \eta_{\bfs_1}}{\partial W_{\ell_1}(k_1,r_1)}\frac{\partial \eta_{\bfs_1}}{\partial W_{\ell_2}(k_2,r_2)} \right],
\end{aligned}
\end{equation}
\end{footnotesize}%
and second derivatives with respect to $b_\ell(r)$ can be computed similarly. With these calculations, the next result is straightforward.
\begin{proposition}
Let $f:\mathbb{R}^{Nd^2 + Nd} \to \mathbb{R}$ be defined as in (\ref{eqn-N-layer}). Then, $\nabla f: \mathbb{R}^{Nd^2 + Nd} \to \mathbb{R}^{Nd^2 + Nd}$ is locally Lipschitz continuous.
\end{proposition}

We now show that $\nabla f$ is \textit{not} $\rho$-order/integrated Lipschitz continuous. We will do so by using a particular sequence akin to the one constructed in Proposition \ref{result-3-layer}.

\begin{proposition}
Let $f:\mathbb{R}^{Nd^2 + Nd} \to \mathbb{R}$ be as in (\ref{eqn-N-layer}) with $N \geq 4$ and $d \in \mathbb{N}$. $\forall \rho \geq 0$, $\nabla f: \mathbb{R}^{Nd^2 + Nd} \to \mathbb{R}^{Nd^2 + Nd}$ is neither $\rho$-integrated Lipschitz continuous nor $\rho$-ordered Lipschitz continuous.
\end{proposition}
\begin{proof}
For $k,r \in \lbrace 1,\ldots, d \rbrace$,
let $E_{kr} \in \mathbb{R}^{d \times d}$ such that all of $E_{kr}$'s entries are zero except $E_{kr}(k,r) = 1$. 
Since the first standard basis vector of $\mathbb{R}^d$ is denoted $\bfs_1$, we let $\bfs_r$ denote the $r^\mathrm{th}$ standard basis vector for $r=2,\ldots,d$.

We now consider the sequence of points $\lbrace x^n \rbrace \subset \mathbb{R}^{Nd^2 + Nd}$ corresponding to $(W_1^n, W_2^n,\ldots,W_N^n, b_1^n,\ldots, b_N^n)$ where $W_1^n = W_2^n = \cdots W_{N-3}^n = W_{N-2}^n = E_{11}$, $W_{N-1}^n = n\log (n) E_{11}$, $W_{N}^n = n^{-1} \bfs_1^\intercal$, and $b_1^n = b_2^n = \cdots = b_N^n = \bfz$. Note, this selection is analogous to the selection in Proposition \ref{result-3-layer}.

To compute the gradient, note
\begin{footnotesize}
\begin{equation}
\begin{aligned}
\frac{\partial \eta_z}{\partial W_\ell(k,r)}
= \begin{cases}
\left( \prod_{i=\ell+1}^N W_i \right) E_{kr} \left( \prod_{i=1}^{\ell-1} W_i \right) z 
+ \sum_{j=1}^{\ell-1} \left( \prod_{i=\ell+1}^N W_i \right) E_{kr} \left( \prod_{i=j+1}^{\ell-1} W_i \right) b_j & \ell < N \\
\mathbf{s}_r^\intercal \left( \prod_{i=1}^{N-1} W_i \right) z + \sum_{j=1}^{N-1} \mathbf{s}_r^\intercal \left( \prod_{i=j+1}^{N-1} W_i \right) b_j & \ell = N,
\end{cases}
\end{aligned}
\end{equation}
\end{footnotesize}%
where we use the same convention for the products specified above.
For our selection of $x^n$,
\begin{small}
\begin{equation}
\frac{\partial \eta_z}{\partial W_\ell(k,r)}
= \begin{cases}
\log(n) \bfs_1^\intercal  E_{kr}E_{11}z & \ell < N-1 \\
n^{-1} \bfs_1^\intercal E_{kr} E_{11} z & \ell = N-1 \\
n \log(n) \bfs_r^\intercal E_{11} z & \ell = N
\end{cases} 
= \begin{cases}
\log(n) \bfs_1^\intercal z & \ell < N - 1, k=r=1 \\
n^{-1} \bfs_1^\intercal z & \ell = N-1, k=r=1 \\
n \log(n) \bfs_1^\intercal z & \ell = N, r=1 \\
0 & \text{otherwise},
\end{cases}
\end{equation}
\end{small}
where for the case $\ell = N$, $k$ is always $1$ as $W_N \in \mathbb{R}^{1 \times d}$.
Similarly,
\begin{small}
\begin{equation}
\frac{\partial \eta_z}{\partial b_{\ell}(r)} 
=\left(\prod_{i=\ell+1}^N W_i^n \right) \bfs_r
= \begin{cases}
\log(n) & \ell < N - 1, r = 1 \\
n^{-1} & \ell = N - 1, r = 1\\
1 & \ell = N, r = 1\\
0 & \text{otherwise}.
\end{cases}
\end{equation}
\end{small}

Using these calculations in (\ref{eqn-N-layer-partial-W}) and (\ref{eqn-N-layer-partial-b}), and noting, for our choices of $x^n$, $\eta_\bfz = 0$ and $\eta_{\bfs_1} = \log(n)$,
\begin{small}
\begin{equation}
\frac{\partial f(x^n)}{\partial W_{\ell}(k,r)} = -\frac{1}{2(1+n)} \times
\begin{cases}
\log(n) & \ell < N - 1, k=r=1 \\
n^{-1} & \ell = N - 1, k = r = 1 \\
n\log(n) & \ell = N, r = 1 \\
0 & \text{otherwise},
\end{cases}
\end{equation}
\end{small}
and
\begin{small}
\begin{equation}
\frac{\partial f(x^n)}{\partial b_{\ell}(r)} = 
\frac{n - 1}{4(n+1)} \times 
\begin{cases}
\log(n) & \ell < N - 1, r = 1\\
n^{-1} & \ell = N - 1, r = 1\\
1 & \ell = N, r = 1\\
0 & \text{otherwise}.
\end{cases}
\end{equation}
\end{small}
Hence, for $n > 2$,
\begin{small}
\begin{align}
\normD{ \nabla f(x^n) }_2^2 
&= \frac{n^2\log(n)^2 + n^{-2} + (N-2) \log(n)^2}{4(1+n)^2} + \frac{(n-1)^2(1 + n^{-2} + (N-2)\log(n)^2)}{16(n+1)^2} \\
&\in \left[\frac{\log(n)^2}{16}, \frac{N \log(n)^2}{2} \right]. \label{eqn-n-layer-grad-bound}
\end{align}
\end{small}%
We now lower bound the $(1,1)$ entry-wise norm of the Hessian with the second derivative of $f$ with respect to $W_{N}(1,1)$. Since $\eta_z$ is linear in $W_{N}(1,1)$, the second derivative of $\eta_z$ with respect to $W_{N}(1,1)$ is $0$. Hence,
\begin{equation}
\normD{ \nabla^2 f(x^n) }_{(1,1)} \geq \left\vert \frac{\partial^2 f(x^n)}{\partial W_{N}(1,1)^2} \right\vert = \frac{n}{2(1+n)^2} n^2 \log(n)^2 \geq \frac{n \log(n)^2}{8}.
\end{equation}

We now choose 
\begin{equation}
\kappa_n = \frac{n \log(n)^2}{8} \left( \frac{2}{N \log(n)^2} \right)^{\rho / 2}.
\end{equation}
With this choice, we check the hypotheses of Lemma \ref{result-local-only}. For all $\rho \geq 0$ and $n > 2$,
\begin{enumerate}
\item $\kappa_n$ diverges. 
\item Using the lower bound in (\ref{eqn-n-layer-grad-bound}), $\kappa_n \normI{\nabla f(x^n)}_2^\rho$ diverges since
\begin{equation}
\kappa_n \normD{ \nabla f(x^n) }_2^\rho 
\geq \frac{n \log(n)^2}{8} \left( \frac{1}{8 N} \right)^{\rho / 2}.
\end{equation}
\item Using the upper bound in (\ref{eqn-n-layer-grad-bound}), $\normI{ \nabla^2 f(x^n)}_{(1,1)} \geq n \log(n)^2 / 8 \geq \kappa_n \normI{\nabla f(x^n)}_2^\rho$.
\end{enumerate}
Thus, $\nabla f$ is \textit{not} $\rho$-order Lipschitz continuous.
\end{proof}

\section{Conclusion} \label{sec-conclusion}
In this brief work, we discuss recent advances in smoothness conditions that are used for optimization methods that \textit{do not} enforce descent, as inspired by applications in deep learning learning.
While the recently proposed, equivalent conditions of $\rho$-order Lipschitz continuity and $\rho$-integrated Lipschitz continuity (with $\rho > 0$) are more general than global Lipschitz continuity, they are not applicable to the gradient of the objective function for training a deep linear neural network of arbitrary depth and dimension. 
On the other hand, the classical condition of the gradient function being locally Lipschitz continuous is applicable to the gradient of this optimization problem. 
Owing to this result, we suggest that the use of $\rho$-order or $\rho$-integrated Lipschitz continuity of the gradient function be demonstrated on the function class of interest before it is used to analyze the behavior of an optimization method on this function class.

%\section*{Declarations}
%The authors did not receive support from any organization for the submitted work. The authors have no relevant financial or non-financial interests to disclose. Regarding data availability, there is no data associated with this manuscript.

\pagebreak 
%\backmatter

%\bmhead{Supplementary information}

%If your article has accompanying supplementary file/s please state so here. 

%Authors reporting data from electrophoretic gels and blots should supply the full unprocessed scans for key as part of their Supplementary information. This may be requested by the editorial team/s if it is missing.

%Please refer to Journal-level guidance for any specific requirements.

%\bmhead{Acknowledgements}

%Acknowledgements are not compulsory. Where included they should be brief. Grant or contribution numbers may be acknowledged.

%Please refer to Journal-level guidance for any specific requirements.

%Some journals require declarations to be submitted in a standardised format. Please check the Instructions for Authors of the journal to which you are submitting to see if you need to complete this section. If yes, your manuscript must contain the following sections under the heading `Declarations':

%\begin{itemize}
%\item Conflict of interest/Competing interests (check journal-specific guidelines for which heading to use)
%\item Ethics approval and consent to participate
%\item Consent for publication
%\item Data availability 
%\end{itemize}

%\noindent
%If any of the sections are not relevant to your manuscript, please include the heading and write `Not applicable' for that section. 

%\begin{appendices}

%\end{appendices}

\bibliography{smoothness_concepts.bib}% common bib file


\begin{thebibliography}{22}
\ifx \bisbn   \undefined \def \bisbn  #1{ISBN #1}\fi
\ifx \binits  \undefined \def \binits#1{#1}\fi
\ifx \bauthor  \undefined \def \bauthor#1{#1}\fi
\ifx \batitle  \undefined \def \batitle#1{#1}\fi
\ifx \bjtitle  \undefined \def \bjtitle#1{#1}\fi
\ifx \bvolume  \undefined \def \bvolume#1{\textbf{#1}}\fi
\ifx \byear  \undefined \def \byear#1{#1}\fi
\ifx \bissue  \undefined \def \bissue#1{#1}\fi
\ifx \bfpage  \undefined \def \bfpage#1{#1}\fi
\ifx \blpage  \undefined \def \blpage #1{#1}\fi
\ifx \burl  \undefined \def \burl#1{\textsf{#1}}\fi
\ifx \doiurl  \undefined \def \doiurl#1{\url{https://doi.org/#1}}\fi
\ifx \betal  \undefined \def \betal{\textit{et al.}}\fi
\ifx \binstitute  \undefined \def \binstitute#1{#1}\fi
\ifx \binstitutionaled  \undefined \def \binstitutionaled#1{#1}\fi
\ifx \bctitle  \undefined \def \bctitle#1{#1}\fi
\ifx \beditor  \undefined \def \beditor#1{#1}\fi
\ifx \bpublisher  \undefined \def \bpublisher#1{#1}\fi
\ifx \bbtitle  \undefined \def \bbtitle#1{#1}\fi
\ifx \bedition  \undefined \def \bedition#1{#1}\fi
\ifx \bseriesno  \undefined \def \bseriesno#1{#1}\fi
\ifx \blocation  \undefined \def \blocation#1{#1}\fi
\ifx \bsertitle  \undefined \def \bsertitle#1{#1}\fi
\ifx \bsnm \undefined \def \bsnm#1{#1}\fi
\ifx \bsuffix \undefined \def \bsuffix#1{#1}\fi
\ifx \bparticle \undefined \def \bparticle#1{#1}\fi
\ifx \barticle \undefined \def \barticle#1{#1}\fi
\bibcommenthead
\ifx \bconfdate \undefined \def \bconfdate #1{#1}\fi
\ifx \botherref \undefined \def \botherref #1{#1}\fi
\ifx \url \undefined \def \url#1{\textsf{#1}}\fi
\ifx \bchapter \undefined \def \bchapter#1{#1}\fi
\ifx \bbook \undefined \def \bbook#1{#1}\fi
\ifx \bcomment \undefined \def \bcomment#1{#1}\fi
\ifx \oauthor \undefined \def \oauthor#1{#1}\fi
\ifx \citeauthoryear \undefined \def \citeauthoryear#1{#1}\fi
\ifx \endbibitem  \undefined \def \endbibitem {}\fi
\ifx \bconflocation  \undefined \def \bconflocation#1{#1}\fi
\ifx \arxivurl  \undefined \def \arxivurl#1{\textsf{#1}}\fi
\csname PreBibitemsHook\endcsname

%%% 1
\bibitem[\protect\citeauthoryear{Cartis et~al.}{2022}]{cartis2022Evaluationa}
\begin{bbook}
\bauthor{\bsnm{Cartis}, \binits{C.}},
\bauthor{\bsnm{Gould}, \binits{N.I.M.}},
\bauthor{\bsnm{Toint}, \binits{P.L.}}:
\bbtitle{Evaluation {{Complexity}} of {{Algorithms}} for {{Nonconvex
  Optimization}}: {{Theory}}, {{Computation}} and {{Perspectives}}}.
\bsertitle{{{MOS-SIAM Series}} on {{Optimization}}}.
\bpublisher{{Society for Industrial and Applied Mathematics}},
\blocation{Philadelphia}
(\byear{2022}).
\doiurl{10.1137/1.9781611976991}
\end{bbook}
\endbibitem

%%% 2
\bibitem[\protect\citeauthoryear{Beck}{2017}]{beck2017Firstorder}
\begin{bbook}
\bauthor{\bsnm{Beck}, \binits{A.}}:
\bbtitle{First-Order Methods in Optimization}.
\bsertitle{{{MOS-SIAM Series}} on {{Optimization}}}.
\bpublisher{SIAM},
\blocation{Philadelphia}
(\byear{2017})
\end{bbook}
\endbibitem

%%% 3
\bibitem[\protect\citeauthoryear{Bertsekas}{2016}]{bertsekas2016Nonlinear}
\begin{bbook}
\bauthor{\bsnm{Bertsekas}, \binits{D.}}:
\bbtitle{Nonlinear {{Programming}}}
vol. \bseriesno{4}.
\bpublisher{Athena Scientific},
\blocation{Nashua, NH}
(\byear{2016})
\end{bbook}
\endbibitem

%%% 4
\bibitem[\protect\citeauthoryear{Nocedal and
  Wright}{2006}]{nocedal2006Numerical}
\begin{bbook}
\bauthor{\bsnm{Nocedal}, \binits{J.}},
\bauthor{\bsnm{Wright}, \binits{S.J.}}:
\bbtitle{Numerical Optimization},
\bedition{2}nd edn.
\bsertitle{Springer {{Series}} in {{Operations Research}} and {{Financial
  Engineering}}}.
\bpublisher{Springer},
\blocation{New York, NY}
(\byear{2006})
\end{bbook}
\endbibitem

%%% 5
\bibitem[\protect\citeauthoryear{Patel et~al.}{2022}]{patel2022Global}
\begin{barticle}
\bauthor{\bsnm{Patel}, \binits{V.}},
\bauthor{\bsnm{Zhang}, \binits{S.}},
\bauthor{\bsnm{Tian}, \binits{B.}}:
\batitle{Global {{Convergence}} and {{Stability}} of {{Stochastic Gradient
  Descent}}}.
\bjtitle{Advances in Neural Information Processing Systems3601436025}
\bvolume{35},
\bfpage{36014}--\blpage{36025}
(\byear{2022})
\doiurl{10.48550/arxiv.2110.01663}
\end{barticle}
\endbibitem

%%% 6
\bibitem[\protect\citeauthoryear{Armijo}{1966}]{armijo1966Minimization}
\begin{barticle}
\bauthor{\bsnm{Armijo}, \binits{L.}}:
\batitle{Minimization of functions having {{Lipschitz}} continuous first
  partial derivatives}.
\bjtitle{Pacific Journal of mathematics}
\bvolume{16}(\bissue{1}),
\bfpage{1}--\blpage{3}
(\byear{1966})
\doiurl{10.2140/pjm.1966.16.1}
\end{barticle}
\endbibitem

%%% 7
\bibitem[\protect\citeauthoryear{Barzilai and
  Borwein}{1988}]{barzilai1988Twopoint}
\begin{botherref}
\oauthor{\bsnm{Barzilai}, \binits{J.}},
\oauthor{\bsnm{Borwein}, \binits{J.M.}}:
Two-point step size gradient methods.
IMA Journal of Numerical Analysis
\textbf{8}(1)
(1988)
\doiurl{10.1093/imanum/8.1.141}
\end{botherref}
\endbibitem

%%% 8
\bibitem[\protect\citeauthoryear{Burdakov
  et~al.}{2019}]{burdakov2019Stabilized}
\begin{barticle}
\bauthor{\bsnm{Burdakov}, \binits{O.}},
\bauthor{\bsnm{Dai}, \binits{Y.}},
\bauthor{\bsnm{Huang}, \binits{N.}}:
\batitle{Stabilized {{Barzilai-Borwein Method}}}.
\bjtitle{Journal of Computational Mathematics}
\bvolume{37}(\bissue{6}),
\bfpage{916}--\blpage{936}
(\byear{2019})
\doiurl{10.4208/jcm.1911-m2019-0171}
\end{barticle}
\endbibitem

%%% 9
\bibitem[\protect\citeauthoryear{Nesterov}{2013}]{nesterov2013Gradient}
\begin{barticle}
\bauthor{\bsnm{Nesterov}, \binits{Y.}}:
\batitle{Gradient methods for minimizing composite functions}.
\bjtitle{Mathematical programming}
\bvolume{140}(\bissue{1}),
\bfpage{125}--\blpage{161}
(\byear{2013})
\doiurl{10.1007/s10107-012-0629-5}
\end{barticle}
\endbibitem

%%% 10
\bibitem[\protect\citeauthoryear{Bauschke et~al.}{2017}]{bauschke2017Descent}
\begin{barticle}
\bauthor{\bsnm{Bauschke}, \binits{H.H.}},
\bauthor{\bsnm{Bolte}, \binits{J.}},
\bauthor{\bsnm{Teboulle}, \binits{M.}}:
\batitle{A descent lemma beyond {{Lipschitz}} gradient continuity: First-order
  methods revisited and applications}.
\bjtitle{Mathematics of Operations Research}
\bvolume{42}(\bissue{2}),
\bfpage{330}--\blpage{348}
(\byear{2017})
\doiurl{10.1287/moor.2016.0817}
\end{barticle}
\endbibitem

%%% 11
\bibitem[\protect\citeauthoryear{Curtis and
  Robinson}{2019}]{curtis2019Exploiting}
\begin{barticle}
\bauthor{\bsnm{Curtis}, \binits{F.E.}},
\bauthor{\bsnm{Robinson}, \binits{D.P.}}:
\batitle{Exploiting negative curvature in deterministic and stochastic
  optimization}.
\bjtitle{Mathematical Programming}
\bvolume{176},
\bfpage{69}--\blpage{94}
(\byear{2019})
\doiurl{10.1007/s10107-018-1335-8}
\end{barticle}
\endbibitem

%%% 12
\bibitem[\protect\citeauthoryear{Malitsky and
  Mishchenko}{2020}]{malitsky2020Adaptive}
\begin{bchapter}
\bauthor{\bsnm{Malitsky}, \binits{Y.}},
\bauthor{\bsnm{Mishchenko}, \binits{K.}}:
\bctitle{Adaptive gradient descent without descent}.
In: \bbtitle{37th {{International Conference}} on {{Machine Learning}},
  {{ICML}} 2020},
vol. \bseriesno{PartF168147-9}
(\byear{2020}).
\doiurl{10.48550/arxiv.1910.09529}
\end{bchapter}
\endbibitem

%%% 13
\bibitem[\protect\citeauthoryear{Wu et~al.}{2020}]{wu2020Wngrad}
\begin{botherref}
\oauthor{\bsnm{Wu}, \binits{X.}},
\oauthor{\bsnm{Ward}, \binits{R.}},
\oauthor{\bsnm{Bottou}, \binits{L.}}:
Wngrad: {{Learn}} the Learning Rate in Gradient Descent
(2020).
\doiurl{10.48550/arXiv.1803.02865}
\end{botherref}
\endbibitem

%%% 14
\bibitem[\protect\citeauthoryear{Grapiglia and
  Stella}{2022}]{grapiglia2022Adaptive}
\begin{barticle}
\bauthor{\bsnm{Grapiglia}, \binits{G.N.}},
\bauthor{\bsnm{Stella}, \binits{G.F.D.}}:
\batitle{An adaptive trust-region method without function evaluations}.
\bjtitle{Computational Optimization and Applications}
\bvolume{82}(\bissue{1}),
\bfpage{31}--\blpage{60}
(\byear{2022})
\doiurl{10.1007/s10589-022-00356-0}
\end{barticle}
\endbibitem

%%% 15
\bibitem[\protect\citeauthoryear{Gratton et~al.}{2022}]{gratton2022Convergence}
\begin{botherref}
\oauthor{\bsnm{Gratton}, \binits{S.}},
\oauthor{\bsnm{Jerad}, \binits{S.}},
\oauthor{\bsnm{Toint}, \binits{P.}}:
Convergence Properties of an {{Objective-Function-Free Optimization}}
  Regularization Algorithm, Including an {{O}}({$E$}{\textasciicircum}3/2)
  Complexity Bound.
arXiv
(2022).
\doiurl{10.48550/arXiv.2203.09947}
\end{botherref}
\endbibitem

%%% 16
\bibitem[\protect\citeauthoryear{Patel and Berahas}{2024}]{patel2024Gradient}
\begin{botherref}
\oauthor{\bsnm{Patel}, \binits{V.}},
\oauthor{\bsnm{Berahas}, \binits{A.S.}}:
Gradient {{Descent}} in the {{Absence}} of {{Global Lipschitz Continuity}} of
  the {{Gradients}}.
SIAM Journal on Mathematics of Data Science,
602--626
(2024)
\doiurl{10.1137/22m1527210}
\end{botherref}
\endbibitem

%%% 17
\bibitem[\protect\citeauthoryear{Varner and Patel}{2024}]{varner2024Challenges}
\begin{botherref}
\oauthor{\bsnm{Varner}, \binits{C.}},
\oauthor{\bsnm{Patel}, \binits{V.}}:
The {{Challenges}} of {{Optimization For Data Science}}
(2024)
\doiurl{10.48550/arxiv.2404.09810}
\end{botherref}
\endbibitem

%%% 18
\bibitem[\protect\citeauthoryear{Chen et~al.}{2023}]{chen2023GeneralizedSmooth}
\begin{bchapter}
\bauthor{\bsnm{Chen}, \binits{Z.}},
\bauthor{\bsnm{Zhou}, \binits{Y.}},
\bauthor{\bsnm{Liang}, \binits{Y.}},
\bauthor{\bsnm{Lu}, \binits{Z.}}:
\bctitle{Generalized-{{Smooth Nonconvex Optimization}} is {{As Efficient As
  Smooth Nonconvex Optimization}}}.
In: \bbtitle{Proceedings of the 40th {{International Conference}} on {{Machine
  Learning}}},
pp. \bfpage{5396}--\blpage{5427}.
\bpublisher{PMLR},
\blocation{Honolulu, HI}
(\byear{2023})
\end{bchapter}
\endbibitem

%%% 19
\bibitem[\protect\citeauthoryear{Li et~al.}{2023}]{li2023Convex}
\begin{botherref}
\oauthor{\bsnm{Li}, \binits{H.}},
\oauthor{\bsnm{Qian}, \binits{J.}},
\oauthor{\bsnm{Tian}, \binits{Y.}},
\oauthor{\bsnm{Rakhlin}, \binits{A.}},
\oauthor{\bsnm{Jadbabaie}, \binits{A.}}:
Convex and {{Non-convex Optimization Under Generalized Smoothness}}.
Advances in Neural Information Processing Systems
\textbf{36}
(2023)
\doiurl{10.48550/arxiv.2306.01264}
\end{botherref}
\endbibitem

%%% 20
\bibitem[\protect\citeauthoryear{Zhang et~al.}{2019}]{zhang2019Why}
\begin{bchapter}
\bauthor{\bsnm{Zhang}, \binits{J.}},
\bauthor{\bsnm{He}, \binits{T.}},
\bauthor{\bsnm{Sra}, \binits{S.}},
\bauthor{\bsnm{Jadbabaie}, \binits{A.}}:
\bctitle{Why {{Gradient Clipping Accelerates Training}}: {{A Theoretical
  Justification}} for {{Adaptivity}}}.
In: \bbtitle{International {{Conference}} on {{Learning Representations}}}
(\byear{2019}).
\doiurl{10.48550/arxiv.1905.11881}
\end{bchapter}
\endbibitem

%%% 21
\bibitem[\protect\citeauthoryear{Zhang et~al.}{2020}]{zhang2020Improved}
\begin{botherref}
\oauthor{\bsnm{Zhang}, \binits{B.}},
\oauthor{\bsnm{Jin}, \binits{J.}},
\oauthor{\bsnm{Fang}, \binits{C.}},
\oauthor{\bsnm{Wang}, \binits{L.}}:
Improved {{Analysis}} of {{Clipping Algorithms}} for {{Non-convex
  Optimization}}.
Advances in Neural Information Processing Systems
\textbf{33}
(2020)
\doiurl{10.48550/arxiv.2010.02519}
\end{botherref}
\endbibitem

%%% 22
\bibitem[\protect\citeauthoryear{Varner and Patel}{2023}]{varner2023Novel}
\begin{barticle}
\bauthor{\bsnm{Varner}, \binits{C.}},
\bauthor{\bsnm{Patel}, \binits{V.}}:
\batitle{A {{Novel Gradient Methodology}} with {{Economical Objective Function
  Evaluations}} for {{Data Science Applications}}}.
\bjtitle{arXiv}
(\byear{2023})
\doiurl{10.48550/arXiv.2309.10894}
\end{barticle}
\endbibitem

\end{thebibliography}

\end{document}